\documentclass{article} 
\usepackage{nips14submit_e,times}
\usepackage{hyperref}
\usepackage{url}

\usepackage{latexsym}
\usepackage{geometry}
\usepackage{graphicx}
\usepackage{hyperref}
\usepackage{subfigure}
\usepackage{epstopdf} 
\usepackage{xcolor}
\usepackage{datetime}

\usepackage{times}
\usepackage{graphicx} 
\usepackage{subfigure} 

\usepackage{algorithm}
\usepackage{algorithmic}
\usepackage{amsmath,amsfonts,amsthm}

\newtheorem{Thm}{Theorem}
\newtheorem{Lem}[Thm]{Lemma}
\newtheorem{Cor}[Thm]{Corollary}

\newcommand{\bbR}{{\mathbf{R}}}

\def\bbR{\mathbb{R}}

\newcommand{\E}[1]{\ensuremath{\mathbb{E}\left[#1\right]}}

\newcommand{\citep}[1]{\cite{#1}}
\newcommand{\citet}[1]{\cite{#1}}

\def\bbR{\mathbb{R}}
\newcommand{\nb}[3]{#1^{(#2)}(#3)}

\def\pr{{\mbox{\rm Pr}}}

\newcommand{\calX}{\mathcal{X}}

\def\calE{\mathcal{E}}

\def\supp{{\mbox{\rm supp}}}

\author{
Kamalika Chaudhuri \\
Computer Science and Engineering Dept.\\
University of California, San Diego\\
\texttt{kamalika@cs.ucsd.edu} \\
\And
Sanjoy Dasgupta \\
Computer Science and Engineering Dept.\\
University of California, San Diego\\
\texttt{dasgupta@cs.ucsd.edu} \\
}

\title{Rates of Convergence for Nearest Neighbor Classification}

\nipsfinalcopy 

\begin{document}

\maketitle
\begin{abstract}
Nearest neighbor methods are a popular class of nonparametric estimators with several desirable properties, such as adaptivity to different distance scales in different regions of space. Prior work on convergence rates for nearest neighbor classification has not fully reflected these subtle properties. We analyze the behavior of these estimators in metric spaces and provide finite-sample, distribution-dependent rates of convergence under minimal assumptions. As a by-product, we are able to establish the universal consistency of nearest neighbor in a broader range of data spaces than was previously known. We illustrate our upper and lower bounds by introducing smoothness classes that are customized for nearest neighbor classification.

\end{abstract}
\def\E{{\mathbb{E}}}
\def\P{{\mathbb{P}}}
\def\R{{\mathbb{R}}}
\def\X{{\mathcal{X}}}
\def\Y{{\mathcal{Y}}}

\section{Introduction}

In this paper, we deal with binary prediction in metric spaces. A classification problem is defined by a metric space $(\X, \rho)$ from which instances are drawn, a space of possible labels $\Y = \{0,1\}$, and a distribution $\P$ over $\X \times \Y$. The goal is to find a function $h: \X \rightarrow \Y$ that minimizes the probability of error on pairs $(X,Y)$ drawn from $\P$; this error rate is the {\it risk} $R(h) = \P(h(X) \neq Y)$. The best such function is easy to specify: if we let $\mu$ denote the marginal distribution of $X$ and $\eta$ the conditional probability $\eta(x) = \P(Y = 1 | X = x)$, then the predictor $1(\eta(x) \geq 1/2)$ achieves the minimum possible risk, $R^* = \E_X [\min(\eta(X), 1-\eta(X))]$. The trouble is that $\P$ is unknown and thus a prediction rule must instead be based only on a finite sample of points $(X_1, Y_1), \ldots, (X_n, Y_n)$ drawn independently at random from $\P$.

Nearest neighbor (NN) classifiers are among the simplest prediction rules. The {\it 1-NN classifier} assigns each point $x \in \X$ the label $Y_i$ of the closest point in $X_1, \ldots, X_n$ (breaking ties arbitrarily, say). For a positive integer $k$, the {\it $k$-NN classifier} assigns $x$ the majority label of the $k$ closest points in $X_1, \ldots, X_n$. In the latter case, it is common to let $k$ grow with $n$, in which case the sequence $(k_n: n \geq 1)$ defines a {\it $k_n$-NN classifier}.

The asymptotic consistency of nearest neighbor classification has been studied in detail, starting with the work of Fix and Hodges~\cite{FH51}. The risk $R_n$ is a random variable that depends on the training sample $(X_1, Y_1), \ldots, (X_n, Y_n)$; the usual order of business is to first determine the limiting behavior of the expected value $\E R_n$ and to then study stronger modes of convergence of $R_n$. Cover and Hart~\cite{CH67} studied the asymptotics of $\E R_n$ in general metric spaces, under the assumption that every $x$ in the support of $\mu$ is either a continuity point of $\eta$ or has $\mu(\{x\}) > 0$. For the 1-NN classifier, they found that $\E R_n \rightarrow \E_X [2\eta(X)(1-\eta(X))] \leq 2R^*(1-R^*)$; for $k_n$-NN with $k_n \uparrow \infty$ and $k_n/n \downarrow 0$, they found $\E R_n \rightarrow R^*$. For points in Euclidean space, a series of results starting with Stone~\cite{S77} established consistency without any distributional assumptions. For $k_n$-NN in particular, $R_n \rightarrow R^*$ almost surely~\citep{DGKL94}.

These consistency results place nearest neighbor methods in a favored category of nonparametric estimators. But for a fuller understanding it is important to also have rates of convergence. For instance, part of the beauty of nearest neighbor is that it appears to adapt automatically to different distance scales in different regions of space. It would be helpful to have bounds that encapsulate this property.

Rates of convergence are also important in extending nearest neighbor classification to settings such as active learning, semisupervised learning, and domain adaptation, in which the training data is not a fully-labeled data set obtained by i.i.d.\ sampling from the future test distribution. For instance, in active learning, the starting point is a set of unlabeled points $X_1, \ldots, X_n$, and the learner requests the labels of just a few of these, chosen adaptively to be as informative as possible about $\eta$. There are many natural schemes for deciding which points to label: for instance, one could repeatedly pick the point furthest away from the labeled points so far, or one could pick the point whose $k$ nearest labeled neighbors have the largest disagreement among their labels. The asymptotics of such selective sampling schemes has been considered in earlier work~\citep{D12}, but ultimately the choice of scheme must depend upon finite-sample behavior. The starting point for understanding this behavior is to first obtain a characterization in the non-active setting.

\subsection{Previous work on rates of convergence}

The earliest rates of convergence for nearest neighbor were distribution-free. Cover~\cite{C68} studied the 1-NN classifier in the case $\X = \R$, under the assumption of class-conditional densities with uniformly-bounded third derivatives. He showed that $\E R_n$ converges at a rate of $O(1/n^2)$. Wagner~\cite{W71} and later Fritz~\cite{F75} also looked at 1-NN, but in higher dimension $\X = \R^d$. The latter obtained an asymptotic rate of convergence for $R_n$ under the milder assumption of non-atomic $\mu$ and lower semi-continuous class-conditional densities.

Distribution-free results are of some value, but fail to precisely characterize which properties of a distribution most influence the performance of nearest neighbor classification. More recent work has investigated several different approaches to obtaining distribution-dependent bounds. Kulkarni and Posner~\cite{KP95} obtained finite-sample rates of convergence for 1-NN and $k_n$-NN in terms of the smoothness of $\eta$. They assumed that for some constants $K$ and $\alpha$, and for all $x_1, x_2 \in \X$,
$$ | \eta(x_1) - \eta(x_2) | \ \leq \ K \rho (x_1, x_2)^{2\alpha} .$$
They then gave bounds in terms of the Holder parameter $\alpha$. Gyorfi~\cite{G81} looked at the case $\X = \R^d$, under the weaker assumption that for some function $K: \R^d \rightarrow \R$ and some $\alpha$, and for all $z \in \R^d$ and all $r > 0$,
$$ \left| \eta(z) - \frac{1}{\mu(B(z,r))} \int_{B(z,r)} \eta(x) \mu(dx) \right| \leq K(z) r^{\alpha} .$$ 
This $\alpha$ is similar in spirit to the earlier Holder parameter, but does not require $\eta$ to be continuous. Gyorfi obtained asymptotic rates in terms of $\alpha$. Another generalization of standard smoothness conditions was proposed recently~\cite{UBS11} in a ``probabilistic Lipschitz'' assumption, and in this setting rates were obtained for NN classification in bounded spaces $\X \subset \R^d$.

The convergence rates obtained so far have been unsatisfactory in several regards. The finite-sample rates require continuity and thus, for instance, do not apply to discrete distributions. The use of a single Holder parameter is convenient but does not capture cases where different regions of the data space have different distance-scales: a common situation in which NN methods might be expected to shine. Most importantly, what is crucial for nearest neighbor is not how $|\eta(x)-\eta(x')|$ scales with $\rho(x,x')$---which is what any Lipschitz or Holder constant captures---but rather how it scales with $\mu(B(x,\rho(x,x')))$. In other words, a suitable smoothness parameter for NN is one that measures the change in $\eta(x)$ with respect to probability mass rather than distance. We will try to make this point clearer in the next section.

\subsection{Some illustrative examples}

We now look at a few examples to get a sense of what properties of a distribution most critically affect the convergence rate of nearest neighbor. In each case, we study the $k$-NN classifier.

As a first example, consider a finite instance space $\X$. For large enough $n$, the $k$ nearest neighbors of a query $x$ will all be $x$ itself, leading immediately to an error bound. However, this kind of reasoning yields an asymptotic rate of convergence. To get a finite-sample rate, we proceed more generally and observe that for any number of points $n$, the $k$ nearest neighbors of $x$ will lie within a ball $B = B(x,r)$ whose probability mass under $\mu$ is roughly $k/n$. 
The quality of the prediction can be assessed by how much $\eta$ varies within this ball. To be slightly more precise, let $\eta(B) = (1/\mu(B)) \int_B \eta(x) \mu(dx)$ denote the average $\eta$ value within the ball. For the $k$-NN prediction at $x$ to be good, we require that if $\eta(x)$ is significantly more than $1/2$ then so is $\eta(B)$; and likewise if $\eta(x)$ is significantly less than $1/2$.

As a second example, consider a distribution over $\X = \R$ in which the two classes ($Y = 0$ and $Y = 1$) have class-conditional densities $\mu_1$ and $\mu_2$, respectively. Assume that these two distributions are supported on disjoint intervals, as shown on the left side of Figure~\ref{fig:1d-1}. Now let's determine the probability that the $k$-NN classifier makes a mistake on a specific query $x$. Clearly, this will happen only if $x$ is near the boundary between the two classes. To be precise, consider an interval around $x$ of probability mass $k/n$, that is, an interval $B = [x-r, x+r]$ with $\mu(B) = k/n$. Then the $k$ nearest neighbors will lie roughly in this interval, and there will likely be an error only if the interval contains a substantial portion of the wrong class. Whether or not $\eta$ is smooth, or the $\mu_i$ are smooth, is irrelevant.

\begin{figure}
\begin{center}
\includegraphics[width=2.25in]{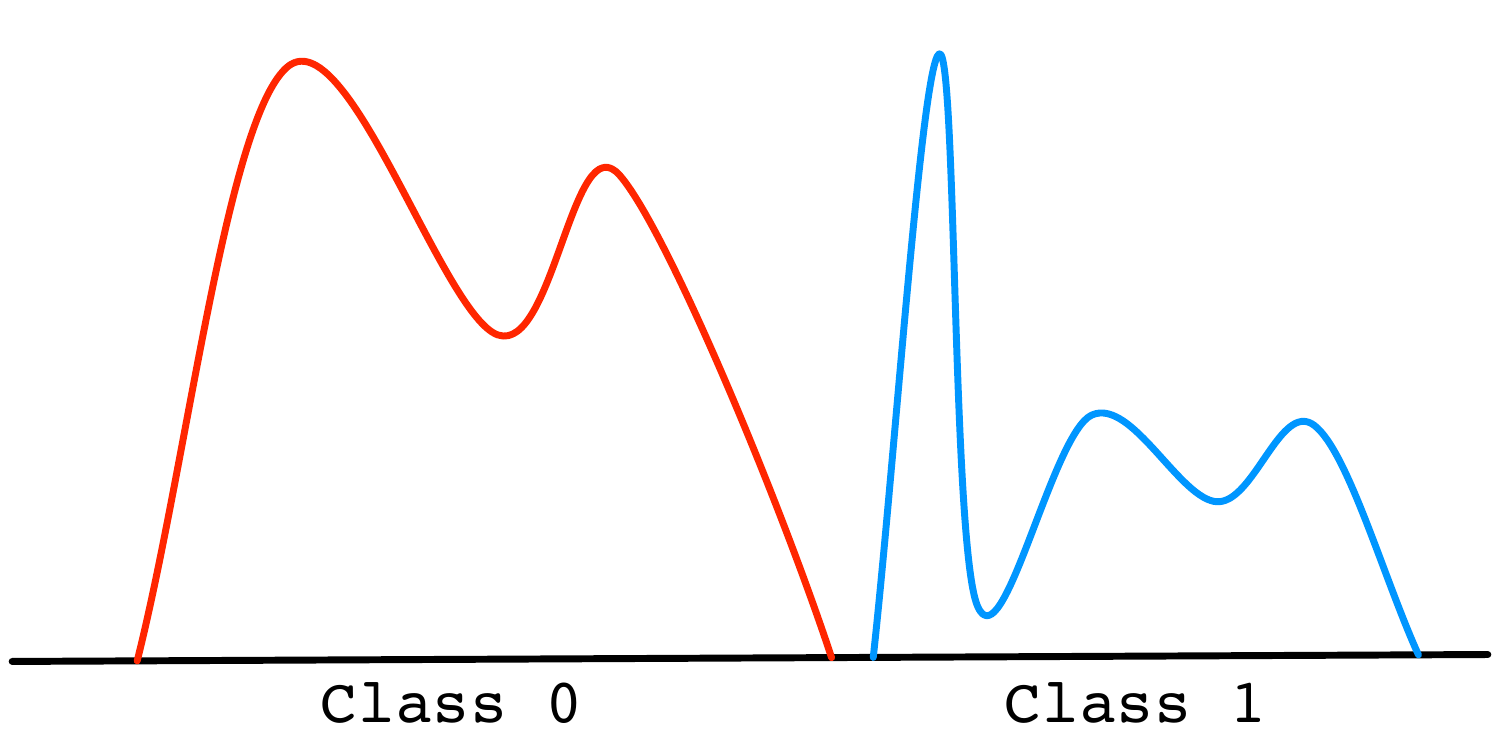}
\hskip.25in
\includegraphics[width=2.25in]{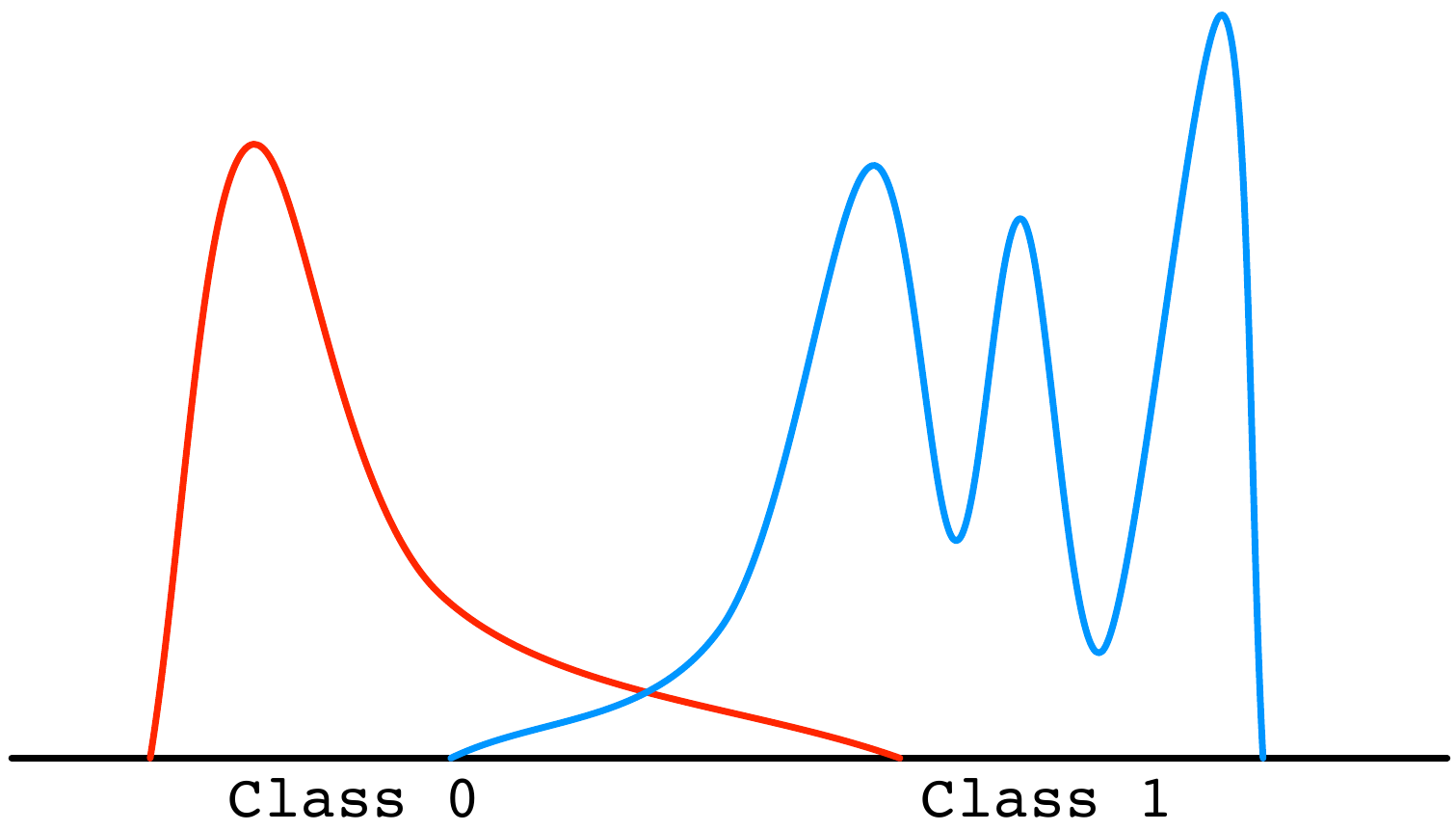}
\end{center}
\caption{One-dimensional distributions. In each case, the class-conditional densities are shown.}
\label{fig:1d-1}
\end{figure}

It should already be clear that the central objects in analyzing $k$-NN are balls of probability mass $\approx k/n$, specifically those near the decision boundary. Now let's see a variant of the previous example (Figure~\ref{fig:1d-1}, right) in which it is no longer the case that $\eta \in \{0,1\}$. Although one of the class-conditional densities in the figure is highly non-smooth, this erratic behavior occurs far from the decision boundary and thus does not affect nearest neighbor performance. And in the vicinity of the boundary, what matters is not how much $\eta$ varies within intervals of any given radius $r$, but rather within intervals of probability mass $k/n$.

These examples hopefully clarify that rates of convergence based only on Holder-continuity of $\eta$---or similar notions---are inadequate for properly characterizing the statistical behavior of nearest neighbor classifiers.

\subsection{Results of this paper}

Let us return to our earlier setting of pairs $(X,Y)$, where $X$ takes values in a metric space $(\X, \rho)$ and has distribution $\mu$, while $Y \in \{0,1\}$ has conditional probability function $\eta(x) = \Pr(Y = 1| X = x)$. We obtain rates of convergence for $k$-NN by attempting to make precise the intuitions discussed above. This leads to a somewhat different style of analysis than has been used in earlier work.

For any positive integers $k \leq n$, we define a notion of {\it effective boundary} for $k$-NN under sample size $n$. For the moment, denote this set by $A_{n,k} \subset \X$.
\begin{itemize}
\item We show that with high probability over the training data, the misclassification rate of the $k$-NN classifier (with respect to the Bayes-optimal classifer) is bounded above by $\mu(A_{n,k})$ plus a small additional term that can be made arbitrarily small.
\item We identify a general condition under which, as $n$ and $k$ grow, $A_{n,k}$ approaches the actual decision boundary $\{x \ |\ \eta(x) = 1/2\}$. This yields universal consistency in a broader range of metric spaces than just $\R^d$.
\item We give a lower bound on the error probability using a different notion of effective boundary.
\item We introduce a Holder-like smoothness condition that is tailored to nearest neighbor. We compare our upper and lower bounds under this kind of smoothness.
\item We obtain risk bounds under the margin condition of Tsybakov that match the best known results for nonparametric classification.
\item We look at additional specific cases of interest: when $\eta$ is bounded away from $1/2$, and the even more extreme scenario where $\eta \in \{0,1\}$ (zero Bayes risk).
\end{itemize}

\section{Definitions and results}

Let $(\calX, \rho)$ denote a separable metric space, and $\mu$ a Borel regular probability measure on this space (that is, open sets are measurable, and every set is contained in a Borel set of the same measure) from which {\it instances} $X$ are drawn. The {\it label} of an instance $X = x$ is $Y \in \{0,1\}$ and is distributed according to the conditional probability function $\eta: \calX \rightarrow [0,1]$ as follows: $\pr(Y = 1 | X = x) = \eta(x)$.

Given a {\it training set} $S = ((X_1, Y_1), \ldots, (X_n, Y_n))$ and a {\it query point} $x \in \calX$, we use the notation $\nb{X}{i}{x}$ to denote the $i$-th nearest neighbor of $x$ in the training set, and $\nb{Y}{i}{x}$ to denote its label. Distances are calculated with respect to the given metric $\rho$, and ties are broken by preferring points earlier in the sequence. The $k$-NN classifier is defined by
$$ g_{n,k}(x) 
\ = \ 
\left\{
\begin{array}{ll}
1 & \mbox{if $\nb{Y}{1}{x} + \cdots + \nb{Y}{k}{x} \geq k/2$} \\
0 & \mbox{otherwise}
\end{array}
\right.
$$
We analyze the performance of $g_{n,k}$ by comparing it with $g(x) = 1(\eta(x) \geq 1/2)$, the omniscent Bayes-optimal classifier. Specifically, we obtain bounds on $\Pr_X(g_{n,k}(X) \neq g(X))$ that hold with high probability over the choice of training data $S$.

\subsection{Definitions}

We begin with some definitions and notation. 

\paragraph{The radius and probability-radius of a ball.} For any $x \in \calX$, let 
$$B^o(x,r) = \{x' \in \calX\ |\ \rho(x,x') < r\} \mbox{\ \ and \ \ } B(x, r) = \{x' \in \calX\ |\  \rho(x,x') \leq r\}$$ 
denote the open and closed balls, respectively, of radius $r$ centered at $x$. We will mostly be dealing with balls that contain a prescribed probability mass. To this end, for any $x \in \calX$ and any $0 \leq p \leq 1$, define
$$ r_p(x) = \inf\{ r\  |\  \mu(B(x, r)) \geq p \} .$$
Thus $\mu(B(x, r_p(x))) \geq p$ (Lemma~\ref{lemma:probability-radius}), and $r_p(x)$ is the smallest radius for which this holds.

\paragraph{The support of $\mu$.} The support of distribution $\mu$ is defined as
$$ \supp(\mu) = \{x \in \calX\ |\ \mu(B(x,r)) > 0 \mbox{\ \ for all $r > 0$}\}.$$
It was shown by \citet{CH67} that in separable metric spaces, $\mu(\supp(\mu)) = 1$. For the interested reader, we reproduce their brief proof in the appendix (Lemma~\ref{lemma:support}).

\paragraph{The conditional probability function for a set.} The conditional probability function $\eta$ is defined for points $x \in \calX$, and can be extended to measurable sets $A \subset \calX$ with $\mu(A) > 0$ as follows:
\begin{equation}
\eta(A) = \frac{1}{\mu(A)} \int_A \eta \, d\mu .
\label{eq:eta-for-sets}
\end{equation}
This is the probability that $Y = 1$ for a point $X$ chosen at random from the distribution $\mu$ restricted to set $A$. We exclusively consider sets $A$ of the form $B(x,r)$, in which case $\eta$ is defined whenever $x \in \supp(\mu)$.

\paragraph{The effective interiors of the two classes, and the effective boundary.} When asked to make a prediction at point $x$, the $k$-NN classifier finds the $k$ nearest neighbors, which can be expected to lie in $B(x,r_p(x))$ for $p \approx k/n$. It then takes an average over these $k$ labels, which has a standard deviation of $\Delta \approx 1/\sqrt{k}$. With this in mind, there is a natural definition for the {\it effective interior} of the $Y = 1$ region: the points $x$ with $\eta(x) > 1/2$ on which the $k$-NN classifier is likely to be correct:
$$ 
\calX_{p,\Delta}^+  
= \{x \in \supp(\mu)\ |\  \eta(x) > \frac{1}{2}, \; \eta(B(x,r)) \geq \frac{1}{2} + \Delta \mbox{\ \ for all $r \leq r_p(x)$}\}.
$$
The corresponding definition for the $Y = 0$ region is
$$ 
\calX_{p,\Delta}^-  
= \{x \in \supp(\mu)\ |\ \eta(x) < \frac{1}{2}, \; \eta(B(x,r)) \leq \frac{1}{2} - \Delta \mbox{\ \ for all $r \leq r_p(x)$}\}.
$$
The remainder of $\calX$ is the effective boundary,
$$ \partial_{p,\Delta} = \calX \setminus (\calX_{p,\Delta}^+ \cup \calX_{p,\Delta}^- ).$$
Observe that $\partial_{p',\Delta'} \subset \partial_{p,\Delta}$ whenever $p' \leq p$ and $\Delta' \leq \Delta$. Under mild conditions, as $p$ and $\Delta$ tend to zero, the effective boundary tends to the actual decision boundary $\{x\ |\ \eta(x) = 1/2\}$ (Lemma~\ref{lemma:effective-boundary-limit}), which we shall denote $\partial_o$.

\subsection{A general bound on the misclassification error}

We begin with a general upper bound on the misclassification rate of the $k$-NN classifier. We will later specialize it to various situations of interest. All proofs appear in the appendix.
\begin{Thm}
Pick any $0 < \delta < 1$ and positive integers $k < n$. Let $g_{n,k}$ denote the $k$-NN classifier based on $n$ training points, and $g(x)$ the Bayes-optimal classifier. With probability at least $1-\delta$ over the choice of training data, 
$$ \pr_X(g_{n,k}(X) \neq g(X)) \ \leq \ \delta + \mu \big(\partial_{p, \Delta}\big) ,$$
where
$$ 
p = \frac{k}{n} \cdot \frac{1}{1 - \sqrt{(4/k) \ln (2/\delta)}}, \mbox{\ \ \ and \ \ \ }
\Delta = \min\left(\frac{1}{2}, \sqrt{\frac{1}{k} \ln \frac{2}{\delta}} \right).
$$
\label{thm:general}
\end{Thm}

Convergence results for nearest neighbor have traditionally studied the excess risk $R_{n,k} - R^*$, 
where $R_{n,k} = \pr(Y \neq g_{n,k}(X))$. If we define the pointwise quantities
\begin{align*}
R_{n,k}(x) &= \pr(Y \neq g_{n,k}(x) | X = x) \\
R^*(x) &= \min(\eta(x), 1- \eta(x)),
\end{align*}
for all $x \in \calX$, we see that
\begin{equation}
R_{n,k}(x) - R^*(x) \ = \  |1 - 2 \eta(x)| 1(g_{n,k}(x) \neq g(x)) .\label{eq:risk}
\end{equation}
Taking expectation over $X$, we then have $R_{n,k} - R^* \leq \pr_X(g_{n,k}(X) \neq g(X))$, and so Theorem~\ref{thm:general} is also an upper bound on the excess risk.

To obtain an asymptotic result, we can take a sequence of integers $(k_n)$ and reals $(\delta_n)$ for which the corresponding $p_n, \Delta_n \downarrow 0$. As we will see, this implies that $\partial_{p_n,\Delta_n}$ converges to the decision boundary, $\partial_o$.

\subsection{Universal consistency}

A series of results, starting with \cite{S77}, has shown that $k_n$-NN is strongly consistent ($R_n = R_{n,k_n} \rightarrow R^*$ almost surely) when $\calX$ is a finite-dimensional Euclidean space and $\mu$ is a Borel measure. A consequence of Theorem~\ref{thm:general} is that this phenomenon holds quite a bit more generally. In fact, strong consistency holds in any metric measure space $(\calX, \rho, \mu)$ for which the Lebesgue differentiation theorem is true: that is, spaces in which, for any bounded measurable $f$, 
\begin{equation}
\lim_{r \downarrow 0} \frac{1}{\mu(B(x,r))} \int_{B(x,r)} f \, d\mu = f(x) 
\label{eq:differentiation}
\end{equation}
for almost all ($\mu$-a.e.) $x \in \calX$.

For more details on this differentiation property, see \cite[2.9.8]{F69} and \cite[1.13]{H01}. It holds, for instance:
\begin{itemize}
\item When $(\X, \rho)$ is a finite-dimensional normed space~\cite[1.15(a)]{H01}.
\item When $(\X, \rho, \mu)$ is {\it doubling}~\cite[1.8]{H01}, that is, when there exists a constant $C(\mu)$ such that $\mu(B(x,2r)) \leq C(\mu) \mu(B(x,r))$ for every ball $B(x,r)$.
\item When $\mu$ is an atomic measure on $\X$.
\end{itemize}

\begin{Thm}
Suppose metric measure space $(\calX, \rho, \mu)$ satisfies differentiation condition (\ref{eq:differentiation}). Pick a sequence of positive integers $(k_n)$, and for each $n$, let $R_n = R_{n,k_n}$ be the risk of the $k_n$-NN classifier $g_{n,k_n}$.
\begin{enumerate}
\item If $k_n \rightarrow \infty$ and $k_n/n \rightarrow 0$, then for all $\epsilon > 0$, 
$$ \lim_{n \rightarrow \infty} \pr_n(R_n - R^* > \epsilon) = 0 .$$ 
Here $\pr_n$ denotes probability over the training set $(X_1, Y_1), \ldots, (X_n, Y_n)$.
\item If in addition $k_n/(\log n) \rightarrow \infty$, then $R_n \rightarrow R^*$ almost surely.
\end{enumerate}
\label{thm:consistency}
\end{Thm}

\subsection{A lower bound}

Next, we give a counterpart to Theorem~\ref{thm:general} that lower-bounds the expected probability of error of $g_{n,k}$. For any positive integers $k < n$, define the high-error set $\calE_{n,k} = \calE_{n,k}^+ \cup \calE_{n,k}^-$, where
\begin{align*}
\calE_{n,k}^+ 
&= 
\left\{x \in \supp(\mu) \ |\  \eta(x) > \frac{1}{2},\ \  \eta(B(x,r)) \leq \frac{1}{2} + \frac{1}{\sqrt{k}}
\mbox{\ for all\ } r_{k/n}(x) \leq r \leq r_{(k+\sqrt{k}+1)/n}(x) \right\} \\
\calE_{n,k}^- 
&= 
\left\{x \in \supp(\mu) \ |\ \eta(x) < \frac{1}{2},\ \  \eta(B(x,r)) \geq \frac{1}{2} - \frac{1}{\sqrt{k}}
\mbox{\ for all\ } r_{k/n}(x) \leq r \leq r_{(k+\sqrt{k}+1)/n}(x) \right\} .
\end{align*}
(Recall the definition (\ref{eq:eta-for-sets}) of $\eta(A)$ for sets $A$.) We will see that for smooth $\eta$ this region is comparable to the effective decision boundary $\partial_{k/n, 1/\sqrt{k}}$. Meanwhile, here is a lower bound that applies to any $(\calX, \rho, \mu)$.
\begin{Thm}
For any positive integers $k < n$, let $g_{n,k}$ denote the $k$-NN classifier based on $n$ training points. There is an absolute constant $c_o$ such that the expected misclassification rate satisfies
$$ \E_n \pr_X(g_{n,k}(X) \neq g(X)) \geq c_o \, \mu(\calE_{n,k}) ,$$
where $\E_n$ is expectation over the choice of training set.
\label{thm:lower-bound}
\end{Thm}

\subsection{Smooth measures}

For the purposes of nearest neighbor, it makes sense to define a notion of smoothness with respect to the marginal distribution on instances. We use a variant of Holder-continuity: for $\alpha, L > 0$, we say the conditional probability function $\eta$ is {\it $(\alpha,L)$-smooth} in metric measure space $(\calX, \rho, \mu)$ if for all $x,x' \in \calX$,
$$ |\eta(x) - \eta(x')| \ \leq \ L \, \mu(B^o(x, \rho(x,x')))^\alpha .$$
This is stated to resemble standard smoothness conditions, but what we will really need is the weaker assertion that for all $x \in \supp(\mu)$ and all $r > 0$,
$$ |\eta(B(x,r)) - \eta(x)| \ \leq \ L \, \mu(B^o(x,r))^\alpha .$$

In such circumstances, the earlier upper and lower bounds on generalization error take on a more easily interpretable form. Recall that the key term in the upper bound (Theorem~\ref{thm:general}) is $\mu(\partial_{p,\Delta})$, for $p \approx k/n$ and $\Delta \approx 1/\sqrt{k}$.
\begin{Lem}
If $\eta$ is $(\alpha,L)$-smooth in $(\calX, \rho, \mu)$, then for any $p, \Delta \geq 0$,
$$ \partial_{p,\Delta} \cap \supp(\mu) \ \subset \ \bigg\{ x \in \calX\ \bigg|\ \big|\eta(x) - \frac{1}{2}\big| \leq \Delta + L p^\alpha \bigg\} .$$
\label{lemma:smooth-upper}
\end{Lem}
This yields a bound on $\pr_X(g_{n,k}(X) \neq g(X))$ that is roughly of the form
$\mu(\{ x \ | \ |\eta(x) - 1/2| \leq k^{-1/2} + L (k/n)^\alpha)$.
The optimal setting of $k$ is then $\sim n^{2\alpha/(2\alpha+1)}$.

The key term in the lower bound of Theorem~\ref{thm:lower-bound} is $\mu(\calE_{n,k})$. Under the smoothness condition, it becomes directly comparable to the upper bound.
\begin{Lem}
If $\eta$ is $(\alpha,L)$-smooth in $(\calX, \rho, \mu)$, then for any $k,n$,
$$ \calE_{n,k} \supset \left\{ x \in \supp(\mu) \ \bigg|\ \eta(x) \neq \frac{1}{2}, \ |\eta(x) - \frac{1}{2}| \leq \frac{1}{\sqrt{k}} - L \left( \frac{k+\sqrt{k}+1}{n} \right)^\alpha \right\} .$$
\label{lemma:smooth-lower}
\end{Lem}

It is common to analyze nonparametric classifiers under the assumption that $\X = \R^d$ and that $\eta$ is {\it $\alpha_H$-Holder continuous} for some $\alpha > 0$, that is,
$$ |\eta(x) - \eta(x')| \leq L \|x- x'\|^{\alpha_H}$$
for some constant $L$. These bounds typically also require $\mu$ to have a density that is uniformly bounded (above and/or below). We now relate these assumptions to our notion of smoothness.
\begin{Lem}
Suppose that $\X \subset \R^d$, and $\eta$ is $\alpha_H$-Holder continuous, and $\mu$ has a density with respect to Lebesgue measure that is $\geq \mu_{\rm \scriptsize min}$ on $\X$. Then there is a constant $L$ such that for any $x \in \mbox{supp}(\mu)$ and $r > 0$ with $B(x,r) \subset \X$,
$$|\eta(x) - \eta(B(x,r))| \leq L \mu(B^o(x,r))^{\alpha_H/d} .$$
\label{lemma:relation-to-Holder}
\end{Lem}
(To remove the requirement that $B(x,r) \subset \X$, we would need the boundary of $\X$ to be well-behaved, for instance by requiring that $\X$ contains a constant fraction of every ball centered in it. This is a familiar assumption in results on nonparametric classification, including the seminal work of \cite{AT07} that we discuss in the next section.)

Our smoothness condition for nearest neighbor problems can thus be seen as a generalization of the usual Holder conditions. It applies in broader range of settings, for example for discrete $\mu$.

\subsection{Margin bounds}
\label{sec:tsybakov}

An achievement of statistical theory in the past two decades has been {\em{margin bounds}}, which give fast rates of convergence for many classifiers when the underlying data distribution $\P$ (given by $\mu$ and $\eta$) satisfies a {\it large margin condition} stipulating, roughly, that $\eta$ moves gracefully away from $1/2$ near the decision boundary.

Following \cite{MT99,T04,AT07}, for any $\beta \geq 0$, we say $\P$ satisfies the {\it $\beta$-margin condition} if there exists a constant $C > 0$ such that
\[ 
\mu \left(\left\{x \ \Big| \ \big{|}\eta(x) - \frac{1}{2}\big{|} \leq t \right\}\right) \ \leq \ C t^{\beta}. 
\]
Larger $\beta$ implies a larger margin. We now obtain bounds for the misclassification rate and the excess risk of $k$-NN under smoothness and margin conditions. 
\begin{Thm}
Suppose $\eta$ is $(\alpha,L)$-smooth in $(\X, \rho, \mu)$ and satisfies the $\beta$-margin condition (with constant $C$), for some $\alpha, \beta, L, C \geq 0$. In each of the two following statements, $k_o$ and $C_o$ are constants depending on $\alpha, \beta, L, C$.
\begin{enumerate}
\item[(a)] For any $0 < \delta < 1$, set $k = k_o n^{2\alpha/(2\alpha+1)} (\log (1/\delta))^{1/(2\alpha+1)}$. With probability at least $1-\delta$ over the choice of training data,
$$ \pr_X(g_{n,k}(X) \neq g(X)) \ \leq \ \delta + C_o \left( \frac{\log (1/\delta)}{n} \right)^{\alpha\beta/(2\alpha+1)}.$$
\item[(b)] Set $k = k_o n^{2\alpha/(2\alpha+1)}$. Then
$\E_n R_{n,k} - R^* \ \leq \ C_o n^{-\alpha(\beta+1)/(2\alpha+1)}.$
\end{enumerate}
\label{thm:margin}
\end{Thm}
It is instructive to compare these bounds with the best known rates for nonparametric classification under the margin assumption. The work of \cite{AT07} (Theorems 3.3 and 3.5) shows that when $(\X, \rho) = (\R^d, \| \cdot \|)$, and $\eta$ is $\alpha_H$-Holder continuous, and $\mu$ lies in the range $[\mu_{\rm\scriptsize min}, \mu_{\rm\scriptsize max}]$ for some $\mu_{\rm\scriptsize max} > \mu_{\rm\scriptsize min} > 0$, and the $\beta$-margin condition holds (along with some other assumptions), an excess risk of $n^{-\alpha_H (\beta+1) / (2 \alpha_H + d)}$ is achievable and is also the best possible. This is exactly the rate achieved by nearest neighbor classification, once we translate between the different notions of smoothness as per Lemma~\ref{lemma:relation-to-Holder}.

Another interesting scenario is when $\eta$ is bounded away from $1/2$, that is, there exists some $\Delta^*$ for which
\[  \mu \left(\left\{x \ \Big| \ \big{|} \eta(x) - \frac{1}{2} \big{|} \leq \Delta^*\right\}\right) = 0 \]
It follows from Lemma~\ref{lemma:smooth-upper} that if $\eta$ is $(\alpha,L)$-smooth in $(\calX, \rho, \mu)$, then $\mu(\partial_{p, \Delta}) = 0$ whenever $\Delta + L p^{\alpha} \leq \Delta^*$. Invoking Theorem~\ref{thm:general} with
\[ k = \frac{n}{2} \left( \frac{\Delta^*}{2L} \right)^{1/\alpha}, \; \; \delta = 2e^{-k (\Delta^*)^2/4}, \]yields an exponentially fast rate of convergence: $\Pr(g_n(X) \neq g(X)) \leq 2 e^{-C_on}$, where $C_o = (\Delta^*)^{2 + 1/\alpha}/(8 (2L)^{1/\alpha})$.

A final case of interest is when $\eta \in \{0,1\}$, so that the Bayes risk $R^*$ is zero. We treat this in Section~\ref{sec:zero-bayes-risk} in the appendix.

\clearpage

\bibliographystyle{plain}
\bibliography{knn}

\clearpage
\section*{Appendix: Analysis}

\subsection{A tie-breaking mechanism}
\label{sec:tie-breaking}

In some situations, such as discrete instance spaces, there is a non-zero probability that two or more of the training points will be equidistant from the query point. In practice, we break ties by a simple rule such as preferring points that appear earlier in the sequence. To accurately reflect this in the analysis, we adopt the following mechanism: for each training point $X$, we also draw a value $Z$ independently and uniformly at random from $[0,1]$. When breaking ties, points with lower $Z$ value are preferred. We use the notation $X' = (X, Z) \in \calX \times [0,1]$ to refer to the {\it augmented} training instances, drawn from the product measure $\mu \times \nu$, where $\nu$ is Lebesgue measure on $[0,1]$.

Given a query point $x \in \calX$ and training points $X_1', \ldots, X_n' \in \calX \times [0,1]$, let $X_{(1)}'(x), \ldots, X_{(n)}'(x)$ denote a reordering of these points by increasing distance from $x$, where each $X_{(i)}'$ is of the form $(X_{(i)}, Z_{(i)})$. With probability 1, this ordering is unambiguous. Also, let $Y_{(1)}(x), \ldots, Y_{(n)}(x)$ be the corresponding labels.

We will need to consider balls in the augmented space. For $x_o \in \calX$, $r_o \geq 0$, and $z_o \in [0,1]$, define
\begin{align*}
B'(x_o, r_o, z_o) 
&= 
\{(x,z) \in \calX \times [0,1]\ |\  \mbox{either $\rho(x_o, x) < r_o$ or ($\rho(x_o, x) = r_o$ and $z < z_o$)}\} \\
&= 
\big(B^o(x_o, r_o) \times [0,1] \big) \bigcup \big((B(x_o, r_o) \setminus B^o(x_o, r_o)) \times [0,z_o) \big) .
\end{align*} 
Given a set of training points $(X_i', Y_i)$ and an augmented ball $B' = B'(x_o, r_o, z_o)$, let $\widehat{Y}(B')$ denote the mean of the $Y_i$ for points $X_i' \in B'$; if there is no $X_i' \in B'$, then this is undefined. 

Let $\eta(B')$ denote the mean probability that $Y=1$ for points $(x,z) \in B'$; formally, it is given by
$$ \eta(B') \ = \ \frac{1}{(\mu \times \nu)(B')} \int_{B'} \eta \, d(\mu \times \nu) $$
whenever $(\mu\times\nu)(B') > 0$. Here $\eta(x,z)$ is defined to be $\eta(x)$.

The ball $B' = B(x_o, r_o, z_o)$ in the augmented space can be thought of as lying between the open ball $B^o = B^o(x_o, r_o)$ and the closed ball $B = B(x_o, r_o)$ in the original space; and indeed $\eta(B')$ is a convex combination of $\eta(B)$ and $\eta(B^o)$ (Lemma~\ref{lemma:augmented-eta}).

\subsection{Proof of Theorem~\ref{thm:general}}
\label{sec:proof-thm-general}

Theorem~\ref{thm:general} rests on the following basic observation.
\begin{Lem}
  Let $g_{n,k}$ denote the $k$-NN classifier based on training data $(X_1', Y_1), \ldots, (X_n', Y_n)$. Pick any $x_o \in \calX$ and any $0 \leq p \leq 1$, $0 \leq \Delta \leq 1/2$. Let $B' = B'(x_o, \rho(x_o, X_{(k+1)}(x_o)), Z_{(k+1)})$. Then
\begin{eqnarray*}
1(g_{n,k}(x_o) \neq g(x_o)) 
&\leq &
1(x_o \in \partial_{p,\Delta}) + \\
& & 1(\rho(x_o, X_{(k+1)}(x_o)) > r_p(x_o)) + \\
&& 1(|\widehat{Y}(B') - \eta(B')| \geq \Delta).
\end{eqnarray*}
\label{lemma:main-inequality}
\end{Lem}
\begin{proof}
Suppose $x_o \not\in \partial_{p, \Delta}$. Then, without loss of generality, $x_o$ lies in $\calX_{p,\Delta}^+$, whereupon $\eta(B(x_o, r)) \geq 1/2 + \Delta$ for all $r \leq r_p(x_o)$.

Next, suppose $r = \rho(x_o, X_{(k+1)}(x_o)) \leq r_p(x_o)$. Then $\eta(B(x_o, r))$ and $\eta(B^o(x_o, r))$ are both $\geq 1/2 + \Delta$ (Lemma~\ref{lemma:open-eta}). By Lemma~\ref{lemma:augmented-eta}, $\eta(B')$ is a convex combination of these and is thus also $\geq 1/2 + \Delta$.

The prediction $g_{n,k}(x_o)$ is based on the average of the $Y_i$ values of the $k$ points closest to $x_o$, in other words, $\widehat{Y}(B')$. If this average differs from $\eta(B')$ by less than $\Delta$, then it is $> 1/2$, whereupon the prediction is correct.
\end{proof}

When we take expectations in the inequality of Lemma~\ref{lemma:main-inequality}, we see that there are three probabilities to be bounded. The second of these, the probability that $\rho(x_o, X_{(k+1)}(x_o)) > r_p(x_o)$, can easily be controlled when $p$ is sufficiently large.
\begin{Lem}
Fix any $x_o \in \calX$ and $0 \leq p, \gamma \leq 1$. Pick any positive integer $k \leq (1-\gamma) np$. Let $X_1, \ldots, X_n$ be chosen uniformly at random from $\mu$. Then
$$ \pr_n(\rho(x_o, X_{(k+1)}(x_o)) > r_p(x_o)) \leq e^{-np\gamma^2/2} \leq e^{-k \gamma^2/2}.$$
\label{lemma:enough-pts-in-ball}
\end{Lem}
\begin{proof}
The probability that any given $X_i$ falls in $B(x_o, r_p(x_o))$ is at least $p$ (Lemma~\ref{lemma:probability-radius}). The probability that $\leq k \leq (1-\gamma)np$ of them land in this ball is, by the multiplicative Chernoff bound, at most $e^{-np\gamma^2/2}$.
\end{proof}

To bound the probability that $\widehat{Y}(B')$ differs substantially from $\eta(B')$, a slightly more careful argument is needed.
\begin{Lem}
Fix any $x_o \in \calX$ and any $0 \leq \Delta \leq 1/2$. Draw $(X_1, Z_1, Y_1), \ldots, (X_n, Z_n, Y_n)$ independently at random and let $B' = B'(x_o, \rho(x_o, X_{(k+1)}(x_o)), Z_{(k+1)}) \subset \calX \times [0,1]$. Then
$$ \pr_n(|\widehat{Y}(B') - \eta(B') | \geq \Delta) 
\ \leq \ 
2 e^{-2k \Delta^2} .
$$
Moreover, if $\eta(B') \in \{0,1\}$ then $\widehat{Y}(B') = \eta(B')$ with probability one.
\label{lemma:y-avg-deviation}
\end{Lem}
\begin{proof}
We will pick the points $X_i' = (X_i, Z_i)$ and their labels $Y_i$ in the following manner:
\begin{enumerate}
\item First pick a point $(X_1, Z_1) \in \calX \times [0,1]$ according to the marginal distribution of the $(k+1)$st nearest neighbor of $x_o$.
\item Pick $k$ points uniformly at random from the distribution $\mu \times \nu$ restricted to $B' = B'(x_o, \rho(x_o, X_1), Z_1)$.
\item Pick $n - k - 1$ points uniformly at random from the distribution $\mu \times \nu$ restricted to $(\calX \times [0,1]) \setminus B'$.
\item Randomly permute the $n$ points obtained in this way.
\item For each $(X_i, Z_i)$ in the permuted order, pick a label $Y_i$ from the conditional distribution $\eta(X_i)$.

\end{enumerate}
The $k$ nearest neighbors of $x_o$ are the points picked in step 2. Their $Y$ values are independent and identically distributed with expectation $\eta(B')$. The main bound in the lemma now follows from a direct application of Hoeffding's inequality.

The final statement of the lemma is trivial and is needed to cover situations in which $\Delta = 1/2$.
\end{proof}

We now complete the proof of Theorem~\ref{thm:general}. Adopt the settings of $p$ and $\Delta$ from the theorem statement, and define the central bad event to be
$$ \mbox{\sc bad}(X_o, X_1', \ldots, X_n', Y_1, \ldots, Y_n) 
\ = \ 
1(\rho(X_o, X_{(k+1)}(X_o)) > r_{p}(X_o)) + 1(|\widehat{Y}(B') - \eta(B')| \geq \Delta) ,$$
where $B'$ is a shorthand for $B'(X_o, \rho(X_o, X_{(k+1)}(X_o)), Z_{(k+1)})$, as before. Fix any $x_o \in \X$. If $\Delta < 1/2$, then by Lemmas~\ref{lemma:enough-pts-in-ball} and \ref{lemma:y-avg-deviation}, 
$$ \E_n \mbox{\sc bad}(x_o, X_1', \ldots, X_n', Y_1, \ldots, Y_n)
\ \leq \ \exp(-k\gamma^2/2) + 2 \exp(-2k\Delta^2)
\ \leq \ \delta^2 ,
$$ 
where $\gamma = 1-(k/np) = \sqrt{(4/k) \ln (2/\delta)}$ and $\E_n$ is expectation over the choice of training data.
If $\Delta = 1/2$ then $\eta(B') \in \{0,1\}$ and we have
$$ \E_n \mbox{\sc bad}(x_o, X_1', \ldots, X_n', Y_1, \ldots, Y_n)
\ \leq \ \exp(-k\gamma^2/2)
\ \leq \ \delta^2 .
$$ 
Taking expectation over $X_o$,
$$ \E_{X_o} \E_n \mbox{\sc bad}(X_o, X_1', \ldots, X_n', Y_1, \ldots, Y_n) \leq \delta^2 ,$$
from which, by switching expectations and applying Markov's inequality, we have
$$\pr_n(\E_{X_o} \mbox{\sc bad}(X_o, X_1', \ldots, X_n', Y_1, \ldots, Y_n) \geq \delta) \leq \delta .$$
The theorem then follows by writing the result of Lemma~\ref{lemma:main-inequality} as
$$ \pr_{X_o}(g_{n,k}(X_o) \neq g(X_o)) \leq \mu(\partial_{p,\Delta}) + \E_{X_o} \mbox{\sc bad}(X_o, X_1', \ldots, X_n', Y_1, \ldots, Y_n).$$

\subsection{Proof of Theorem~\ref{thm:consistency}}

Recall that we define $R_n = \pr_X(g_{n,k_n}(X) \neq Y)$. From equation (\ref{eq:risk}), we have:
$$ R_n - R^* \ \leq \ \pr_X(\eta(X) \neq 1/2 \mbox{\ and\ } g_{n,k_n}(X) \neq g(X)) .$$
Defining $\partial_o = \{x \in \calX\ |\  \eta(x) = 1/2 \}$ to be the decision boundary,
we then have the following corollary of Theorem~\ref{thm:general}.
\begin{Cor}
Let $(\delta_n)$ be any sequence of positive reals, and $(k_n)$ any sequence of positive integers. For each $n$, define $(p_n)$ and $(\Delta_n)$ as in Theorem~\ref{thm:general}. Then
$$ 
\pr_n\big(R_n - R^* > \delta_n + \mu\big(\partial_{p_n, \Delta_n}\setminus \partial_o\big) \big) \leq \delta_n,$$
where $\pr_n$ is probability over the choice of training data.
\label{cor:general}
\end{Cor}

For the rest of the proof, assume that $(\calX, \rho, \mu)$ satisfies Lebesgue's differentiation theorem: that is, for any bounded measurable $f: \calX \rightarrow \bbR$, 
$$ \lim_{r \downarrow 0} \frac{1}{\mu(B(x,r))} \int_{B(x,r)} f \, d\mu = f(x) $$
for almost all ($\mu$-a.e.) $x \in \calX$. We'll see that, as a result, $\mu(\partial_{p_n, \Delta_n} \setminus \partial_o) \rightarrow 0$.

\begin{Lem}
There exists $\calX_o \subset \calX$ with $\mu(\calX_o) = 0$, such that any $x \in \calX \setminus \calX_o$ with $\eta(x) \neq 1/2$ lies in $\calX_{p,\Delta}^+ \cup \calX_{p,\Delta}^-$ for some $p, \Delta > 0$.
\label{lemma:effective-boundary-limit}
\end{Lem}
\begin{proof}
As a result of the differentiation condition,
\begin{equation}
\lim_{r \downarrow 0} \eta(B(x,r)) 
\ = \ 
\lim_{r \downarrow 0} \frac{1}{\mu(B(x,r))} \int_{B(x,r)} \eta \, d\mu 
\ = \ 
\eta(x)
\label{eq:lebesgue}
\end{equation}
for almost all ($\mu$-a.e.) $x \in \calX$. Let $\calX_o$ denote the set of $x$'s for which (\ref{eq:lebesgue}) fails to hold or that are outside $\supp(\mu)$. Then, $\mu(\calX_o) = 0$.

Now pick any $x \not\in \calX_o$ such that $\eta(x) \neq 1/2$. Without loss of generality, $\eta(x) > 1/2$. Set $\Delta = (\eta(x) - 1/2)/2 > 0$. By (\ref{eq:lebesgue}), there is some $r_o > 0$ such that $\eta(B(x,r)) \geq 1/2 + \Delta$ whenever $0 \leq r \leq r_o$. Define $p = \mu(B(x,r_o)) > 0$. Then $r_p(x) \leq r_o$ and $x \in \calX_{p,\Delta}^+$.
\end{proof}

\begin{Lem}
If $p_n, \Delta_n \downarrow 0$, then
$$\lim_{n \rightarrow \infty} \mu \big( \partial_{p_n, \Delta_n} \setminus \partial_o \big) = 0 .$$
\label{lemma:boundary-convergence}
\end{Lem}
\begin{proof}
Let $A_n = \partial_{p_n, \Delta_n} \setminus \partial_o$. Then $A_1 \supset A_2 \supset A_3 \supset \cdots$. We've seen earlier that for any $x \in \calX \setminus (\calX_o \cup \partial_o)$ (where $\calX_o$ is defined in Lemma~\ref{lemma:effective-boundary-limit}), there exist $p, \Delta > 0$ such that $x \not\in \partial_{p, \Delta}$. Therefore,
$$ \bigcap_{n \geq 1} A_n \ \subset \ \calX_o ,$$
whereupon, by continuity from above, $\mu(A_n) \rightarrow 0$.
\end{proof}

Convergence in probability follows immediately.
\begin{Lem}
If $k_n \rightarrow \infty$ and $k_n/n \rightarrow 0$, 
then for any $\epsilon > 0$, 
$$ \lim_{n \rightarrow \infty} \pr_n(R_n - R^* > \epsilon) = 0.$$
\label{lemma:p-convergence}
\end{Lem}
\begin{proof}
First define $\delta_n = \exp(-k_n^{1/2})$, and define the corresponding $p_n, \Delta_n$ as in Theorem~\ref{thm:general}. It is easily checked that the three sequences $\delta_n, p_n, \Delta_n$ all go to zero.

Pick any $\epsilon > 0$. By Lemma~\ref{lemma:boundary-convergence}, we can choose a positive integer $N$ so that $\delta_n \leq \epsilon/2$ and $\mu(\partial_{p_n,\Delta_n} \setminus \partial_o) \leq \epsilon/2$ whenever $n \geq N$. Then by Corollary~\ref{cor:general}, for $n \geq N$,
$$ \pr_n(R_n - R^* > \epsilon) \leq \delta_n .$$
Now take $n \rightarrow \infty$.
\end{proof}

We finish with almost sure convergence.
\begin{Lem}
Suppose that in addition to the conditions of Lemma~\ref{lemma:p-convergence}, we have $k_n/(\log n) \rightarrow \infty$. Then $R_n \rightarrow R^*$ almost surely.
\label{lemma:as-convergence}
\end{Lem}
\begin{proof}
Choose $\delta_n = 1/n^2$, and for each $n$ set $p_n, \Delta_n$ as in Theorem~\ref{thm:general}. It can be checked that the resulting sequences $(p_n)$ and $(\Delta_n)$ both go to zero.

Pick any $\epsilon > 0$. Choose $N$ so that $\sum_{n \geq N} \delta_n \leq \epsilon$. Letting $\omega$ denote a realization of an infinite training sequence $(X_1, Y_1), (X_2, Y_2), \ldots$, we have from Corollary~\ref{cor:general} that
$$ \pr \big\{\omega\ \big|\  \exists n \geq N: R_n(\omega) - R^* > \delta_n + \mu\big(\partial_{p_n,\Delta_n} \setminus \partial_o\big)\big\} \leq \sum_{n\geq N} \delta_n \leq \epsilon .$$
Therefore, with probability at least $1-\epsilon$ over $\omega$, we have
$$ R_n(\omega) - R^* \leq \delta_n + \mu\big(\partial_{p_n,\Delta_n} \setminus \partial_o\big) $$
for all $n \geq N$, whereupon, by Lemma~\ref{lemma:boundary-convergence}, $R_n(\omega) \rightarrow R^*$. The result follows since this is true for any $\epsilon > 0$.
\end{proof}

\subsection{Proof of Theorem~\ref{thm:lower-bound}}

For positive integer $n$ and $0 \leq p \leq 1$, let $\mbox{bin}(n,p)$ denote the (binomial) distribution of the sum of $n$ independent Bernoulli($p$) random variables.  We will use $\mbox{bin}(n,p; \geq k)$ to denote the probability that this sum is $\geq k$; and likewise $\mbox{bin}(n,p; \leq k)$.

It is well-known that the binomial distribution can be approximated by a normal distribution, suitably scaled. Slud~\cite{Slud77} has finite-sample results of this form that will be useful to us.
\begin{Lem}
Pick any $0 < p \leq 1/2$ and any nonnegative integer $\ell$.
\begin{enumerate}
\item[(a)] \cite[p. 404, item (v)]{Slud77} If $\ell \leq np$, then 
$\mbox{\rm bin}(n,p; \geq \ell) \geq 1 - \Phi((\ell - np)/\sqrt{np})$.
\item[(b)] \cite[Thm 2.1]{Slud77} If $np \leq \ell \leq n(1-p)$, then
$\mbox{\rm bin}(n,p; \geq \ell) \geq 1 - \Phi((\ell - np)/\sqrt{np(1-p)})$.
\end{enumerate}
Here $\Phi(a) = (2\pi)^{-1/2}\int_{-\infty}^a \exp(-t^2/2) dt$ is the cumulative distribution function of the standard normal.
\label{lemma:slud}
\end{Lem}

Now we begin the proof of Theorem~\ref{thm:lower-bound}. Fix any integers $k < n$, and any $x_o \in \calE_{n,k}$. Without loss of generality, $\eta(x_o) < 1/2$.

Pick $X_1, \ldots, X_n$ and $Z_1, \ldots, Z_n$ (recall the discussion on tie-breaking in Section~\ref{sec:tie-breaking}) in the following manner:
\begin{enumerate}
\item First pick a point $(X_1, Z_1) \in \calX \times [0,1]$ according to the marginal distribution of the $(k+1)$st nearest neighbor of $x_o$.
\item Pick $k$ points uniformly at random from the distribution $\mu \times \nu$ restricted to $B' = B'(x_o, \rho(x_o, X_1), Z_1)$; recall the earlier definition of the augmented space $\calX \times [0,1]$ and augmented balls within this space.
\item Pick $n - k - 1$ points uniformly at random from the distribution $\mu \times \nu$ restricted to $(\calX \times [0,1]) \setminus B'$.
\item Randomly permute the $n$ points obtained in this way.
\end{enumerate}

The $(k+1)$st nearest neighbor of $x_o$, denoted $X_{(k+1)}(x_o)$, is the point chosen in the first step. With constant probability, it lies within a ball of probability mass $(k+\sqrt{k}+1)/n$ centered at $x_o$, but not within a ball of probability mass $k/n$. Call this event $G_1$:
$$ G_1: \ \ \ \ \ r_{k/n}(x_o) \leq \rho(x_o, X_{(k+1)}(x_o)) \leq r_{(k + \sqrt{k} +1)/n}(x_o)$$
\begin{Lem}
There is an absolute constant $c_1 > 0$ such that $\pr(G_1) \geq c_1$.
\label{lemma:lower-1}
\end{Lem}
\begin{proof}
The expected number of points $X_i$ that fall in $B(x_o, r_{(k+\sqrt{k}+1)/n}(x_o))$ is $\geq k + \sqrt{k} + 1$; the probability that the actual number is $\leq k$ is at most $\mbox{bin}(n, (k + \sqrt{k} + 1)/n; \leq k)$. Likewise, the expected number of points that fall in $B^o(x_o, r_{k/n}(x_o))$ is $\leq k$, and the probability that the actual number is $\geq k+1$ is at most $\mbox{bin}(n, k/n; \geq k+1)$. If neither of these bad events occurs, then $G_1$ holds. Therefore,
$$ \pr(G_1) 
\ \geq \ 
1 - \mbox{bin}\bigg(n,\frac{k + \sqrt{k} + 1}{n}; \ \leq k\bigg) 
- \mbox{bin}\bigg(n, \frac{k}{n}; \ \geq k+1\bigg).$$ 
The last term is easy to bound: it is $\leq 1/2$ since $k$ is the median of $\mbox{bin}(n,k/n)$ \citep{KB80}. To bound the first term, we use Lemma~\ref{lemma:slud}(a):
\begin{align*}
\mbox{bin}\bigg(n,\frac{k + \sqrt{k}+1}{n}; \ \leq k\bigg)
&=
1 - \mbox{bin}\bigg(n,\frac{k + \sqrt{k}+1}{n}; \ \geq k+1\bigg)\\
&\leq
\Phi \bigg( \frac{(k+1)-(k + \sqrt{k}+1)}{\sqrt{k + \sqrt{k}+1}}\bigg)
\ \leq \ 
\Phi(-1/\sqrt{3}),
\end{align*}
which is $1/2 - c_1$ for some constant $c_1 > 0$. Thus $\pr(G_1) \geq c_1$.
\end{proof}

Next, we lower-bound the probability that (conditional on event $G_1$), the $k$ nearest neighbors of $x_o$ have an average $Y$ value with the wrong sign. Recalling that $\eta(x_o) < 1/2$, define the event
$$ G_2: \ \ \ \ \ \widehat{Y}(B') > 1/2 $$
where as before, $B'$ denotes the ball $B'(x_o, X_{(k+1)}(x_o), Z_{k+1})$ in the augmented space. 

\begin{Lem}
There is an absolute constant $c_2 > 0$ such that $\pr(G_2 | G_1) \geq c_2$.
\label{lemma:lower-2}
\end{Lem}
\begin{proof}
Event $G_1$ depends only on step 1 of the sampling process. Assuming this event occurs, step 2 consists in drawing $k$ points from the distribution $\mu \times \nu$ restricted to $B'$. Since $x_o \in \calE_{n,k}$, we have (by an application of Lemmas~\ref{lemma:augmented-eta} and \ref{lemma:open-eta}) that $\eta(B') \geq 1/2 - 1/\sqrt{k}$. Now, $\widehat{Y}(B')$ follows a $\mbox{bin}(k,\eta(B'))$ distribution, and hence, by Lemma~\ref{lemma:slud}(b),
$$
\pr\bigg(\widehat{Y}(B') > \frac{k}{2}\bigg)
\ = \ 
\pr\bigg(\widehat{Y}(B') \geq \bigg\lceil \frac{k+1}{2} \bigg\rceil \bigg) 
\ \geq \  
\pr \bigg(Z \geq \frac{2 \sqrt{k} + 2}{\sqrt{k}} \bigg)
$$
where $Z$ is a standard normal. The last tail probability is at least some constant $c_2$.
\end{proof}

In summary, for $x_o \in \calE_{n,k}$, 
$$ \pr_n (g_{n,k}(x_o) \neq g(x_o)) \ \geq \ \pr(G_1 \wedge G_2) \ \geq \ c_1c_2.$$
Taking expectation over $x_o$, we then get
$$ \E_n \pr_X(g_{n,k}(x) \neq g(x)) \ \geq \ c_1 c_2 \mu(\calE_{n,k}),$$
as claimed.

\subsection{Proofs of Lemmas~\ref{lemma:smooth-upper} and \ref{lemma:smooth-lower}}

It is immediate that if $\eta$ is $(\alpha,L)$-smooth in $(\calX, \rho, \mu)$, then for all $x \in \supp(\mu)$ and all $r > 0$, 
\begin{equation}
|\eta(B(x,r)) - \eta(x)| \ \leq \ L \, \mu(B^o(x,r))^\alpha 
\label{eq:smoothness}
\end{equation}

Pick any $x \in \supp(\mu)$ and any $p \geq 0$. For $r \leq r_p(x)$, we have $\mu(B^o(x,r)) \leq p$ and thus, by (\ref{eq:smoothness}),
$$ |\eta(B(x,r)) - \eta(x)| \ \leq \ Lp^\alpha .$$

As a result, if $\eta(x) > 1/2 + \Delta + Lp^\alpha$ then $\eta(B(x,r)) > 1/2 + \Delta$ whenever $r \leq r_p(x)$. Therefore, such an $x$ lies in the effective interior $\calX_{p,\Delta}^+$. A similar result applies to $x$ with $\eta(x) < 1/2 - \Delta - Lp^\alpha$. Therefore, the boundary region $\partial_{p,\Delta}$ can only contain points $x$ for which $|\eta(x) - 1/2| \leq \Delta + Lp^\alpha$, as claimed by Lemma~\ref{lemma:smooth-upper}.

A similar argument yields Lemma~\ref{lemma:smooth-lower}. Any point $x \in \supp(\mu)$ with
$$ \frac{1}{2} < \eta(x) \leq \frac{1}{2} + \frac{1}{\sqrt{k}} - L \left( \frac{k+\sqrt{k}+1}{n} \right)^\alpha $$
has $\eta(B(x,r)) \leq 1/2 + 1/\sqrt{k}$ for all $r \leq r_{(k+\sqrt{k}+1)/n}(x)$, and therefore lies in $\calE_{n,k}^+$. Likewise for $\calE_{n,k}^-$.

\subsection{Proof of Lemma~\ref{lemma:relation-to-Holder}}

Suppose that $\eta$ satisfies the $\alpha$-Holder condition so that for some constant $C > 0$,
$$ |\eta(x) - \eta(x')| \ \leq \ C \|x-x'\|^{\alpha_H}$$
whenever $x,x' \in \X$. For any $x \in \mbox{supp}(\mu)$ and $r > 0$, we then have
$$ |\eta(x) - \eta(B(x,r))| \ \leq \ C r^{\alpha_H} .$$
If $\mu$ has a density that is lower-bounded by $\mu_{\rm \scriptsize min}$, and $B(x,r) \subset \X$, we also have
$$ \mu(B^o(x,r)) \ \geq \ \mu_{\rm \scriptsize min} v_d r^d ,$$
where $v_d$ is the volume of the unit ball in $\R^d$. The lemma follows by combining these two inequalities.

\subsection{Proof of Theorem~\ref{thm:margin}}

Assume that $\eta$ is $(\alpha,L)$-smooth in $(\X, \rho, \mu)$, that is, 
\begin{equation}
|\eta(B(x,r)) - \eta(x) | \leq L \mu(B^o(x,r))^\alpha \label{eq:smooth-def}
\end{equation}
for all $x \in \mbox{supp}(\mu)$ and all $r > 0$, and also that it satisfies the $\beta$-margin condition (with constant $C$), under which, for any $t \geq 0$,
\begin{equation}
\mu \left(\left\{x \ \Big| \ \big{|}\eta(x) - \frac{1}{2}\big{|} \leq t \right\}\right) \ \leq \ C t^{\beta} .
\label{eq:margin-def}
\end{equation}

\subsubsection*{Proof of Theorem~\ref{thm:margin}(a)}

Set $p, \Delta$ as specified in Theorem~\ref{thm:general}. It follows from that theorem and from Lemma~\ref{lemma:smooth-upper} that under (\ref{eq:smooth-def}) and (\ref{eq:margin-def}), for any $\delta > 0$, with probability at least $1-\delta$ over the choice of training data,
$$ \pr_X(g_{n,k}(X) \neq g(X)) \ \leq \ \delta + \mu(\partial_{p,\Delta}) \ \leq \ \delta + C(\Delta + Lp^\alpha)^\beta.$$
Expanding $p,\Delta$ in terms of $k$ and $n$, this becomes
$$ 
\pr_X(g_{n,k}(X) \neq g(X)) \ \leq \ \delta + C \left( \left(\frac{\ln (2/\delta)}{k}\right)^{1/2} + L \left( \frac{2k}{n} \right)^\alpha \right)^\beta ,$$
provided $k \geq 16 \ln (2/\delta)$. The result follows by setting $k \propto n^{2\alpha/(2\alpha+1)} (\log (1/\delta))^{1/(2\alpha+1)}$.

\subsubsection*{Proof of Theorem~\ref{thm:margin}(b)}

Theorem~\ref{thm:margin}(b) is an immediate consequence of Lemma~\ref{lemma:expected-risk} below. We begin, however, with an intermediate result about the pointwise expected risk.

Fix any $n$ and any $k < n$, and set $p = 2k/n$. Define
\begin{align*}
\Delta(x) &= |\eta(x) - 1/2| \\
\Delta_o  &= Lp^\alpha
\end{align*}
Recall that the Bayes classifier $g(x)$ has risk $R^*(x) = \min(\eta(x), 1-\eta(x))$ at $x$. The pointwise risk of the $k$-NN classifier $g_{n,k}$ is denoted $R_{n,k}(x)$.
\begin{Lem}
Pick any $x \in \supp(\mu)$ with $\Delta(x) > \Delta_o$. Under (\ref{eq:smooth-def}),
$$ \E_n R_{n,k}(x) - R^*(x) \ \leq \ \exp(-k/8) + 4 \Delta(x) \exp(-2k(\Delta(x) - \Delta_o)^2).$$
\label{lemma:pointwise-expected-risk}
\end{Lem}
\begin{proof}
Assume without loss of generality that $\eta(x) > 1/2$. By (\ref{eq:smooth-def}), for any $0 \leq r \leq r_p(x)$, we have
$$ \eta(B(x,r)) \geq \eta(x) - Lp^\alpha = \eta(x) - \Delta_o = \frac{1}{2} + (\Delta(x) - \Delta_o) ,$$
whereby $x \in \X^+_{p, \Delta(x)-\Delta_o}$ (and thus $x \not\in \partial_{p,\Delta(x)-\Delta_o}$).

Next, recalling (\ref{eq:risk}), and then applying Lemma~\ref{lemma:main-inequality},
\begin{align*}
R_{n,k}(x) - R^*(x) 
&= 
2 \Delta(x) 1(g_{n,k}(x) \neq g(x)) \\
&\leq
2 \Delta(x) \left( 1(\rho(x, X_{(k+1)}(x)) > r_p(x)) + 1(|\widehat{Y}(B') - \eta(B')| \geq \Delta(x) - \Delta_o) \right),
\end{align*}
where $B'$ is as defined in that lemma statement. We can now take expectation over the training data and invoke Lemmas~\ref{lemma:enough-pts-in-ball} and \ref{lemma:y-avg-deviation} to conclude
\begin{align*}
\E_n R_{n,k}(x) - R^*(x) 
&\leq 
2 \Delta(x) \left( \pr_n(\rho(x, X_{(k+1)}(x)) > r_p(x)) + \pr_n(|\widehat{Y}(B') - \eta(B')| \geq \Delta(x) - \Delta_o) \right)  \\
&\leq
2 \Delta(x) \left( \exp\left( -\frac{k}{2} \left( 1 - \frac{k}{np}\right)^2 \right) + 2 \exp\left(-2k (\Delta(x) - \Delta_o)^2 \right)\right),
\end{align*}
from which the lemma follows by substituting $p = 2k/n$ and observing $\Delta(x) \leq 1/2$.
\end{proof}

\begin{Lem}
Under (\ref{eq:smooth-def}) and (\ref{eq:margin-def}),
$$ \E_n R_{n,k} - R^* \ \leq \ \exp(-k/8) + 6C \max \left( 2L \left( \frac{2k}{n}\right)^\alpha, \sqrt{\frac{8(\beta+2)}{k}}\right)^{\beta+1} .$$
\label{lemma:expected-risk}
\end{Lem}
\begin{proof}
Recall the definitions of $p (=2k/n)$ and $\Delta_o, \Delta(x)$ above. Further, for each integer $i \geq 1$, define $\Delta_i = \Delta_o \cdot 2^i$. Fix any $i_o \geq 1$.

Lemma~\ref{lemma:pointwise-expected-risk} bounds the expected pointwise risk for any $x$ with $\Delta(x) > \Delta_o$. We will apply it to points with $\Delta(x) > \Delta_{i_o}$. For all remaining $x$, we have $\E_n R_{n,k}(x) - R^*(x) \leq 2\Delta_{i_o}$. Taking expectation over $X$,
\begin{eqnarray*}
\lefteqn{\E_n R_n - R^*} \\
&\leq&
\E_X\left[ 2\Delta_{i_o} 1(\Delta(X) \leq \Delta_{i_o}) + \exp(-k/8) + 4 \Delta(X) \exp(-2k(\Delta(X) - \Delta_o)^2) 1(\Delta(X) > \Delta_{i_o}) \right] \\
&\leq&
2\Delta_{i_o} \pr_X(\Delta(X) \leq \Delta_{i_o}) + \exp(-k/8) + 4 \E_X\left[\Delta(X) \exp(-2k(\Delta(X) - \Delta_o)^2) 1(\Delta(X) > \Delta_{i_o}) \right] .
\end{eqnarray*} 
By the margin condition (\ref{eq:margin-def}), we have $\pr_X(\Delta(X) \leq t) \leq Ct^\beta$. Thus only the last expectation remains to be bounded. We do so by considering each interval $\Delta_i < \Delta(X) \leq \Delta_{i+1}$ separately:
\begin{align}
\lefteqn{\E_X\left[\Delta(X) \exp(-2k(\Delta(X) - \Delta_o)^2) 1(\Delta_i < \Delta(X) \leq \Delta_{i+1}) \right]} \nonumber \\
&\leq 
\Delta_{i+1} \exp(-2k(\Delta_i - \Delta_o)^2) \pr_X(\Delta(X) \leq \Delta_{i+1}) \nonumber \\
&\leq 
C \Delta_{i+1}^{\beta+1} \exp(-2k(\Delta_i - \Delta_o)^2) .
\label{eq:terms}
\end{align}
If we set
$$ i_o \ = \ \max \left( 1, \left\lceil \log_2 \sqrt{\frac{2(\beta+2)}{k \Delta_o^2}} \right\rceil \right) ,$$
then for $i \geq i_o$, the terms (\ref{eq:terms}) are upper-bounded by a geometric series with ratio $1/2$. This is because the ratio of two successive terms can be bounded as
\begin{align*}
\frac{C \Delta_{i+1}^{\beta+1}\exp(-2k(\Delta_i - \Delta_o)^2)}{C \Delta_i^{\beta+1}\exp(-2k(\Delta_{i-1} - \Delta_o)^2)}
&=
2^{\beta+1} \exp(-2k((2^i \Delta_o- \Delta_o)^2 - (2^{i-1} \Delta_o - \Delta_o)^2)) \\
&=
2^{\beta+1} \exp(-2k \Delta_o^2 ((2^i - 1)^2 - (2^{i-1} - 1)^2)) \\
&\leq 
2^{\beta+1} \exp(-2^{2i-1} k \Delta_o^2) \\
&\leq
2^{\beta+1} \exp(-(\beta+2)) \ \leq \ 1/2 .
\end{align*}
Therefore
\begin{align*}
\lefteqn{\E_X\left[\Delta(X) \exp(-2k(\Delta(X) - \Delta_o)^2) 1(\Delta(X) > \Delta_{i_o}) \right]} \\
&=
\sum_{i \geq i_o} \E_X\left[\Delta(X) \exp(-2k(\Delta(X) - \Delta_o)^2) 1(\Delta_i < \Delta(X) \leq \Delta_{i+1}) \right] \\
&\leq 
\sum_{i \geq i_o} C \Delta_{i+1}^{\beta+1} \exp(-2k(\Delta_i - \Delta_o)^2) \\
&\leq
C \Delta_{i_o}^{\beta+1} \exp(-2k(\Delta_{i_o-1} - \Delta_o)^2) 
\ \leq \ 
C \Delta_{i_o}^{\beta+1} .
\end{align*}
Putting these together, we have $\E_n R_{n,k} - R^* \leq 6C \Delta_{i_o}^{\beta+1} + e^{-k/8}$. We finish by substituting $\Delta_{i_o} = 2^{i_o} \Delta_o$.
\end{proof}

\subsection{Zero Bayes Risk}
\label{sec:zero-bayes-risk}

An interesting case is when there is no inherent uncertainty in the conditional probability distribution $p(y|x)$. Formally, for all $x$ in the sample space $\calX$, except those in a subset $\calX_0$ of measure zero, $\eta(x)$ is either $0$ or $1$. In this case, the omniscient Bayes classifier will incur risk $R^* = 0$; however, a classifier based on a finite sample that is unaware of the true $\eta$ will incur some non-zero classification error.

An interesting quantity to consider in this case is the {\em{effective interiors}} of the classes as a whole:
\begin{eqnarray*}
\calX_{p}^+  
= \{x \in \supp(\mu)\ |\  \eta(x) = 1, \eta(B(x,r)) = 1 \mbox{\ \ for all $r \leq r_p(x)$}\}. \\
\calX_{p}^- 
= \{x \in \supp(\mu)\ |\  \eta(x) = 0, \eta(B(x,r)) = 0 \mbox{\ \ for all $r \leq r_p(x)$}\}. 
\end{eqnarray*}

Thus, $\calX_{p}^+ = \calX_{p, 1/2}^+$, and $\calX_{p}^- = \calX_{p, 1/2}^-$. The rest of $\calX$ is the {\em{effective boundary}} between the two classes:
\[  \partial_{p} = \calX \setminus (\calX_{p}^+ \cup \calX_{p}^- ). \]

Incorporating these two quantities into Theorem~\ref{thm:general} yields a bound of the following form.

\begin{Lem}
Let $\delta$ be any positive real and let $k < n$ be positive integers. With probability $\geq 1 - \delta$ over the choice of the training data, the error of the $k$-nearest neighbor classifier $g_{n,k}$ is bounded as:
\[ \Pr_X(g_{n,k}(X) \neq g(X)) \leq \delta + \mu(\partial_p), \]
where 
\[ p = \frac{k}{n} + \frac{2 \log(2/\delta)}{n}\left(1 + \sqrt{1 + \frac{k}{\log(2/\delta)}}\right) \] 
\label{lem:zerobayesrisk}
\end{Lem}

\begin{proof}
The proof is the same as that of Theorem~\ref{thm:general}, except that the probability of the central bad event is different. We will therefore bound the probabilities of these events. 

Observe that under the conditions of the lemma, for any $p$, 
\[ \partial_{p, \frac{1}{2}} = \partial_p \]

Moreover, for any $x_0 \notin \partial_p$, $\eta(B'(x_0, r_p(x_0)))$ is either $0$ or $1$; this implies that for all $x \in B'(x_0, r_p(x_0))$ except those in a measure zero subset, $\eta(x)$ is either $0$ or $1$. Therefore, the probability $\Pr(\hat{Y}(B') \neq \eta(B'))$ is zero.

The rest of the lemma follows from plugging this fact in to the proof of Theorem~\ref{thm:general} and some simple algebra.
\end{proof}

In particular, observe that since $p$ increases with increasing $k$, and the dependence on $\Delta$ is removed, the best bounds are achieved at $k = 1$ for:
\[ p = \frac{1}{n} + \frac{2(1 + \sqrt{2}) \log(2/\delta)}{n} \]
This corroborates the admissibility results of~\cite{CH67}, which essentially state that there is no $k > 1$ such that the $k$-nearest neighbor algorithm has equal or better error than the $1$-nearest neighbor algorithm against all distributions.

\subsection{Additional technical lemmas}

\begin{Lem}
For any $x \in \calX$ and $0 \leq p \leq 1$, we have $\mu(B(x,r_p(x))) \geq p$.
\label{lemma:probability-radius}
\end{Lem}
\begin{proof}
Let $r^* = r_p(x) = \inf\{r\ |\ \mu(B(x,r)) \geq p\}$. For any $n \geq 1$, let $B_n = B(x, r^* + 1/n)$. Thus $B_1 \supset B_2 \supset B_3 \supset \cdots$, with $\mu(B_n) \geq p$. Since $B(x,r^*) = \bigcap_n B_n$, it follows by continuity from above of probability measures that $\mu(B_n) \rightarrow \mu(B(x,r^*))$, so this latter quantity is $\geq p$.
\end{proof}

\begin{Lem}[Cover-Hart]
$\mu(\mbox{\rm supp}(\mu)) = 1$.
\label{lemma:support}
\end{Lem}
\begin{proof}
Let $\calX_o$ denote a countable dense subset of $\calX$. Now, pick any point $x \not\in \mbox{supp}(\mu)$; then there is some $r > 0$ such that $\mu(B(x,r)) = 0$. It is therefore possible to choose a ball $B_x$ centered in $\calX_o$, with rational radius, such that $x \in B_x$ and $\mu(B_x) = 0$. Since there are only countably many balls of this sort, 
$$ 
\mu(\calX \setminus \mbox{supp}(\mu)) 
\ \leq \ 
\mu\big(\bigcup_{x \not\in \mbox{\rm\scriptsize supp}(\mu)} B_x\big) 
\ = \ 
0 .$$
\end{proof}

\begin{Lem}
Pick any $x_o \in \mbox{\rm supp}(\mu)$, $r_o > 0$, and any Borel set $I \subset [0,1]$. Define $B^o = B^o(x_o, r_o)$ and $B = B(x_o, r_o)$ to be open and closed balls centered at $x_o$, and let $A \subset \calX \times [0,1]$ be given by $A = (B^o \times [0,1]) \bigcup ((B\setminus B^o) \times I)$. Then
$$ \eta(A) = \frac{\mu(B) \nu(I)}{\mu(B) \nu(I) + \mu(B^o)(1-\nu(I))} \eta(B) + \frac{\mu(B^o)(1-\nu(I))}{\mu(B) \nu(I) + \mu(B^o)(1-\nu(I))} \eta(B^o).$$
\label{lemma:augmented-eta}
\end{Lem}
\begin{proof}
Since $x_o$ lies in the support of $\mu$, we have $(\mu \times \nu)(A) \geq \mu(B^o) > 0$; hence $\eta(A)$ is well-defined.
\begin{align*}
\eta(A) 
&=
\frac{1}{(\mu\times\nu)(A)} \int_A \eta \ d(\mu \times \nu) \\
&=
\frac{1}{\mu(B^o) + \mu(B \setminus B^o) \nu(I)} \left( \int_{B^o} \eta\, d\mu + \int_{B\setminus B^o} \nu(I) \eta \, d\mu \right) \\
&= 
\frac{1}{\mu(B^o) + (\mu(B) - \mu(B^o)) \nu(I)} \left( \int_{B^o} \eta \, d\mu + \nu(I) \left(\int_B \eta \, d\mu - \int_{B^o} \eta \, d\mu \right) \right)\\
&= 
\frac{\mu(B^o) \eta(B^o) + \nu(I)(\mu(B) \eta(B) - \mu(B^o) \eta(B^o))}{\mu(B^o)(1-\nu(I)) + \mu(B) \nu(I)},
\end{align*}
as claimed.
\end{proof}

\begin{Lem}
Suppose that for some $x_o \in \supp(\mu)$ and $r_o > 0$ and $q > 0$, it is the case that $\eta(B(x_o, r)) \geq q$ whenever $r \leq r_o$. Then $\eta(B^o(x_o, r_o)) \geq q$ as well.
\label{lemma:open-eta}
\end{Lem}
\begin{proof}
Let $B^o = B^o(x_o, r_o)$. Since
$$ B^o = \bigcup_{r < r_o} B(x_o, r) ,$$
it follows from the continuity from below of probability measures that
$$ \lim_{r \uparrow r_o} \mu(B(x_o, r)) = \mu(B^o) $$
and by dominated convergence that
$$ \lim_{r \uparrow r_o} \int_{B(x_o, r)} \eta \, d\mu 
=  \lim_{r \uparrow r_o} \int_{B^o} 1(x \in B(x_o, r)) \eta(x) \mu(dx)
= \int_{B^o} \eta \, d\mu .
$$ 
For any $r \leq r_o$, we have $\eta(B(x_o, r)) \geq q$, which can be rewritten as
$$ \int_{B(x_o,r)} \eta\, d\mu - q \, \mu(B(x_o,r)) \geq 0 .$$
Taking the limit $r \uparrow r_o$, we then get the desired statement.
\end{proof}

\end{document}